%% file: main.tex
\documentclass[10pt,twocolumn,letterpaper]{article}

\usepackage[pagenumbers]{cvpr} 
\usepackage{bm}
\usepackage{algorithmic,algorithm}
\usepackage{multirow} 
\usepackage{booktabs, makecell}
\newcommand{\et}[2]{#1$^{\scriptscriptstyle\pm#2}$}
\newcommand{\ets}[2]{\underline{#1$^{\scriptscriptstyle\pm#2}$}}
\newcommand{\etb}[2]{\textbf{#1$^{\scriptscriptstyle\pm#2}$}}
\usepackage{pifont}
\newcommand{\cmark}{\ding{51}}

\usepackage{amsthm}
\usepackage{mathtools}

\input{preamble}
\definecolor{cvprblue}{rgb}{0.21,0.49,0.74}
\usepackage[pagebackref,breaklinks,colorlinks,allcolors=cvprblue]{hyperref}

\newtheorem{theorem}{Theorem}[section]
\newtheorem{lemma}[theorem]{Lemma}
\newtheorem{proposition}[theorem]{Proposition}
\newtheorem{definition}[theorem]{Definition}
\newtheorem{assumption}[theorem]{Assumption}
\newtheorem{remark}[theorem]{Remark}

\newcommand{\bx}{\mathbf{x}}

\newcommand{\bz}{\mathbf{z}}

\def\paperID{3914} 
\def\confName{CVPR}
\def\confYear{2026}

\title{FloodDiffusion: Tailored Diffusion Forcing for Streaming Motion Generation}

\author{Yiyi Cai\textsuperscript{1}\quad 
Yuhan Wu\textsuperscript{2}\quad 
Kunhang Li\textsuperscript{2} \quad 
You Zhou\textsuperscript{1} \quad 
Bo Zheng\textsuperscript{1} \quad 
Haiyang Liu\textsuperscript{2} \quad 
\\
\textsuperscript{1}Shanda AI Research Tokyo \quad
\textsuperscript{2}The University of Tokyo \quad \\
\\
\url{https://shandaai.github.io/FloodDiffusion/}
}

\begin{document}
\twocolumn[{%
\renewcommand\twocolumn[1][]{#1}%
\maketitle
\begin{center}
    \centering
    \captionsetup{type=figure}
    \vspace{-1cm}
    \includegraphics[trim=0 0 0 0, clip,width=1.0\textwidth]{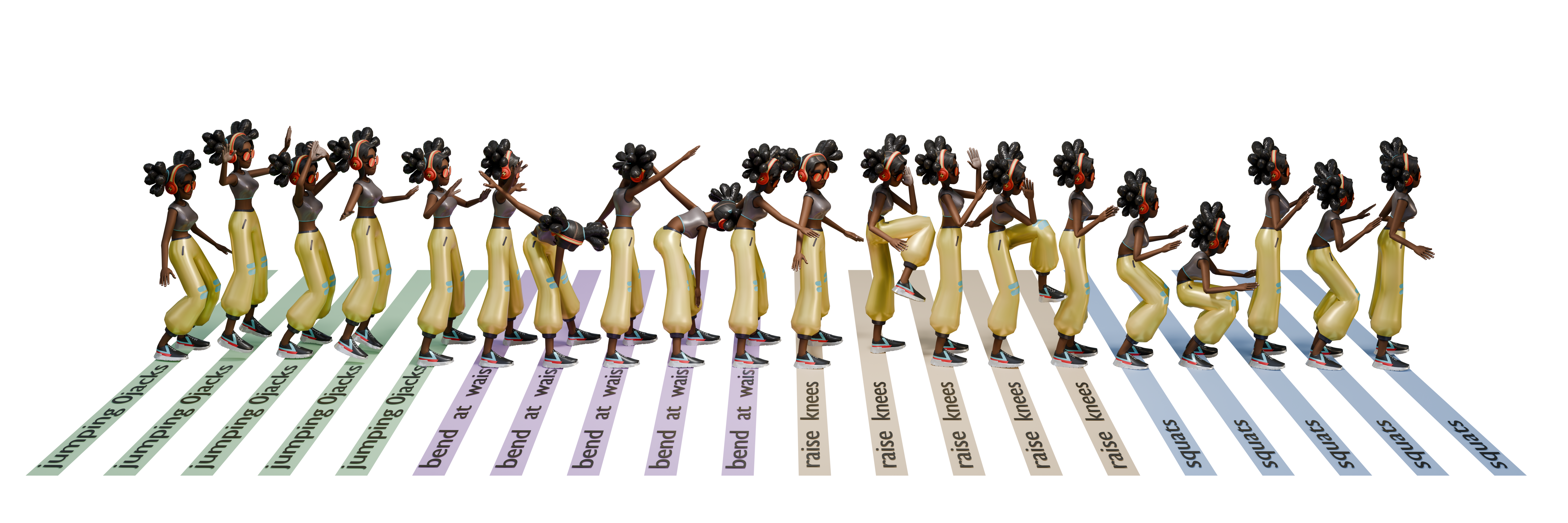}
    \vspace{-0.8cm}
    \captionof{figure}{\textbf{FloodDiffusion} is a diffusion forcing based framework for streaming human motion generation. Given time-varying text prompts, such as ``raise knees'' followed by ``squats'', it generates smooth, continuous human motions aligned with the text. The framework natively handles prompt changes and does not require inference-time optimizations like explicit prompt refresh detection.}
    \label{fig:teaser}
    
\end{center}%
}]

\begin{abstract}
We present \textit{FloodDiffusion}, a new framework for text-driven, streaming human motion generation. Given time-varying text prompts, FloodDiffusion generates text-aligned, seamless motion sequences with real-time latency.
Unlike existing methods that rely on chunk-by-chunk or auto-regressive model with diffusion head, we adopt a diffusion forcing framework to model this time-series generation task under time-varying control events.
We find that a straightforward implementation of vanilla diffusion forcing (as proposed for video models) fails to model real motion distributions. We demonstrate that to guarantee modeling the output distribution, the vanilla diffusion forcing must be tailored to: (i) train with a bi-directional attention instead of casual attention; (ii) implement a lower triangular time scheduler instead of a random one; (iii) utilize a continues time-varying way to introduce text conditioning.
With these improvements, we demonstrate in the first time that the diffusion forcing-based framework achieves state-of-the-art performance on the streaming motion generation task, reaching an FID of 0.057 on the HumanML3D benchmark. Models, code, and weights are available.

\end{abstract}

\section{Introduction}
\label{sec:introduction}

Streaming motion generation from text prompts has growing interest due to its potential for controlling real-time agents, such as in-game NPCs or robotic systems \citep{kodaira2025streamdiffusion,liu2022disco,yang2025shallow,liu2024tango,liu2025video,jiang2023motiongpt}. However, most existing works focus on \textit{non-streaming} motion generation (often noted as text-to-motion \citep{tevet2022human,guo2022tm2t,zhang2024motiondiffuse,petrovich2022temos,ren2023diffusion}), which generates a complete motion sequence from a static text prompt. Directly concatenating outputs from such models for streaming purposes is non-viable, as it results in unsmooth transitions and significant latency.

More recent works have begun to explore dedicated streaming frameworks, primarily based on chunk-by-chunk diffusion \citep{zhang2025primal,xie2025progressive} or auto-regressive (AR) models with a diffusion head \citep{xiao2025motionstreamer,yin2025slow}. For instance, PRIMAL \citep{zhang2025primal} generates sequences via chunk-by-chunk diffusion, conditioned on past motion and the current text. MotionStreamer \citep{xiao2025motionstreamer} uses an AR structure to handle long-term relationships, generating a single token at each step via a small diffusion head. However, these approaches have limitations: (i) chunk-by-chunk diffusion suffers from high ``first-token'' latency, as the generation process must wait to fill a full context length; and (ii) AR-based frameworks are limited in their ability to capture and utilize explicit history from past motions.

In the context of temporal sequence generation, another framework, \textit{diffusion forcing}, was originally proposed for video generation \citep{ho2022video,yang2023diffusion,xie2025progressive,yin2025slow}. By assigning different time-steps to each noisy frame, it theoretically offers advantages in both ``first-token'' latency and the direct utilization of explicit history frames. These advantages directly address the primary limitations of existing methods. Thus, in this paper, we explore a diffusion forcing approach for the streaming human motion generation task.

We found that a straightforward implementation following the vanilla diffusion forcing fails to generate high-quality results. Thus, we introduce a tailored version with several modifications. Firstly, the design of causal attention is replaced with bi-directional attention, to guarantee that frames in the buffer are denoised based on the newest text prompt. Secondly, while vanilla diffusion forcing samples random timesteps for each frame, we demonstrate that a simplified, low triangular based timestep sampler works better. Lastly, to address inconsistent information fusion when new text prompts arrive, we replace the refresh mechanism with a time-varying text conditioning fusion approach.

With these improvements, we demonstrate for the first time that a diffusion forcing-based framework can achieve state-of-the-art performance in this domain. Our final framework, \textit{FloodDiffusion}, achieves an FID of $0.057$ on the HumanML3D dataset. This result outperforms existing streaming motion generation models and performs on par with SOTA non-streaming methods.

In summary, our contributions are:
\begin{itemize}
    \item We introduce \textit{FloodDiffusion}, to our knowledge, the first streaming motion generation framework based on diffusion forcing, featuring modifications in model architecture, training scheduler, and condition fusion.
    \item We mathematically proof that our tailored framework guarantees the reproduction of the target data distribution, similar to original full-length diffusion, without optimizing an ELBO proxy.
    \item We show that \textit{FloodDiffusion} achieves state-of-the-art performance on the streaming motion generation task, verified on the HumanML3D and BABEL dataset.
\end{itemize}

\section{Related Works}
\label{sec:relatedworks}

\subsection{Streaming Generation with Diffusion Forcing}
\label{sec:diffusionandflow}
Diffusion forcing was first described as a way to let different frames/tokens carry different noise levels so that generation can proceed in a flexible manner \citep{chen2024diffusion}. Follow-up works reduce the train–test gap by explicitly rolling out the model during training, \textit{e.g.}, self forcing \citep{huang2025self}, and by keeping a longer denoising buffer \citep{liu2025rolling}. Recent works \citep{sun2025ar, chen2025skyreels} impose a non-decreasing timestep constraint to enforce causal ordering, while still using random schedule. There are also system-level pipelines for real-time interactive diffusion that focus on latency and cache reuse \citep{kodaira2025streamdiffusion}. These studies are mostly in video form, assume spatially large inputs, and use attention masks designed for video. Our task is different: motion is 1D in time, the control (text) may change at any step, and the model must immediately reflect the newest text on the buffered frames. Directly using the video-style diffusion forcing leads to suboptimal motion quality.

\subsection{Human Motion Generation}
\label{sec:humanmotion}
Text-to-motion has been mainly studied in the non-streaming setting. Early works show that motion and text can be modeled jointly, either by tokenizing motion or by learning a shared latent space \citep{guo2022tm2t,petrovich2022temos}. Diffusion-based motion models (\textit{e.g.}, MDM, MotionDiffuse) bring in the diffusion based methods on HumanML3D and KIT-ML \citep{tevet2022human,zhang2024motiondiffuse,ren2023diffusion}. Large or unified models further integrate text-to-motion, motion-to-text, and editing \citep{jiang2023motiongpt}, and discrete/token-based designs like T2M-GPT and MoMask improve compactness and FID \citep{zhang2023generating,guo2024momask,zhang2023remodiffuse,chen2024text}. These methods are usually trained and evaluated on HumanML3D and KIT-ML \citep{guo2022generating,plappert2016kit} and assume the whole text prompt of future motion is available. More recent works start to consider online/streaming motion. MotionStreamer builds a causal latent space and generates motion step by step \citep{xiao2025motionstreamer}. PRIMAL adopts an interactive, avatar-oriented formulation and repeatedly uses a chunk diffusion to extend motion \citep{zhang2025primal}. They can run online, but they either suffer from first-token latency (need to fill a chunk) or only implicitly use long motion history. Our work show that a diffusion-forcing style model, once tailored for motion, can reach the quality of non-streaming text-to-motion models while staying streamable.

\begin{figure*}[t]
    \centering
    \includegraphics[width=\linewidth]{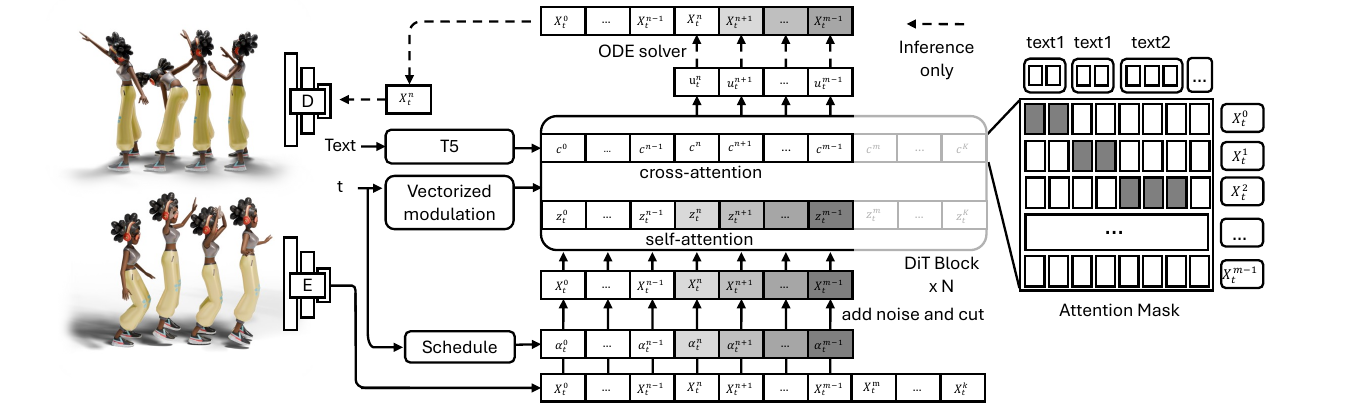}
    \caption{\textbf{Pipeline Overview.} FloodDiffusion is a latent diffusion based framework, the $263D$ motion stream is encoded to a compact $4D$ latent sequence via our causal VAE. Then the model predicts the velocity for the latent, $\hat{u}_t$, for the active window $m(t){:}n(t)$ conditioned on the context $0{:}n(t)$. The key designs are adding noise for the sequence according to lower-triangular time schedule and a Frame-wise text conditioning using an attention mask. During inference, we start from noise and slide the window, generating latent frames that are immediately decoded for streaming output.}
    \label{fig:pipeline}
\end{figure*}

\section{Methods}
\label{sec:methods}

We formulate streaming motion generation as a conditional time-series problem. Given control signals $\mathbf{c}^{0:K}$, \textit{i.e.}, the text prompt in this paper, a generator produces  $\mathbf{X}^{0:K}=g(\mathbf{c}^{0:K})$ in a streaming manner.

\subsection{Preliminaries}
\label{sec:preliminaries}

We fix the initialization distribution to be standard white Gaussian noise:
\begin{equation}
p_{\text{init}} = \mathcal{N}(\mathbf{0},\mathbf{I})
\end{equation}
Diffusion models then perform distribution matching by transporting $p_0 \sim p_{\text{init}}$ to the data distribution $p_T \sim p_{\text{data}}$ via a time-indexed Gaussian corruption path. For each data point $\mathbf{z} \sim p_{\text{data}}$ and time $t\in[0,T]$, define
\begin{equation}
p_t(\mathbf{x} \mid \mathbf{z}) = \mathcal{N}\big(\mathbf{x};\, \alpha_t \, \mathbf{z}, \, \beta_t^2 \, \mathbf{I}\big),
\end{equation}
where $\alpha_t$ and $\beta_t$ are scalar schedules satisfying the boundary conditions $\alpha_0=0$, $\alpha_T=1$, $\beta_0=1$, $\beta_T=0$. This induces a Gaussian conditional probability path from $p_{\text{init}}$ to $p_{\text{data}}$.

\subsection{Vectorized Time Schedule}
\label{sec:timestep}

In the standard formulation, $\alpha_t$ and $\beta_t$ are scalar functions of $t$. We extend them to vector-valued schedules to enable streaming inference.

\begin{definition}[Vectorized Time Schedule]
\label{def:vectorized_schedule}
Let $K$ denote the sequence length. We define vectorized time schedules as:
\begin{align}
\bm{\alpha}_t &= [\alpha_t^0,\alpha_t^1,\dots,\alpha_t^{K-1}] \in \mathbb{R}^K \\
\bm{\beta}_t &= [\beta_t^0,\beta_t^1,\dots,\beta_t^{K-1}] \in \mathbb{R}^K
\end{align}
where $t \in [0, T]$ and each component satisfies the boundary conditions:
\begin{align}
\alpha_0^k = 0, \quad \alpha_T^k &= 1, \quad \beta_0^k = 1, \quad \beta_T^k = 0 \\
\forall k &\in \{0,\dots,K-1\}
\end{align}
\end{definition}

Under this vectorized formulation, the conditional path factorizes across positions:
\begin{equation}
p_t(\mathbf{x}\mid \mathbf{z}) = \prod_{k=0}^{K-1} \mathcal{N}\big(x^k; \alpha_t^k z^k, (\beta_t^k)^2\mathbf{I}\big),
\end{equation}

\begin{proposition}[Vectorized Conditional Dynamics]
\label{prop:conditional_dynamics}
Given the vectorized time schedule, the conditional vector field and score function are:
\begin{align}
u_t(\mathbf{x}\mid \mathbf{z}) &= \left(\dot{\bm{\alpha}}_t - \frac{\dot{\bm{\beta}}_t}{\bm{\beta}_t}\odot \bm{\alpha}_t\right) \odot \mathbf{z} + \left(\frac{\dot{\bm{\beta}}_t}{\bm{\beta}_t}\right) \odot \mathbf{x} \label{eq:conditional_vf}\\
s_t(\mathbf{x}\mid \mathbf{z}) &= -\frac{(\mathbf{x}-\bm{\alpha}_t \odot \mathbf{z})}{\bm{\beta}_t^2} \label{eq:conditional_score}
\end{align}
where $\odot$ denotes element-wise multiplication the division here is also element-wise.
\end{proposition}

\begin{definition}[Marginal Dynamics]
\label{def:marginal_dynamics}
The marginal vector field and score function conditioned on control signal $\mathbf{c}$ are defined as:
\begin{align}
u_t(\mathbf{x},\mathbf{c}) &= \int u_t(\mathbf{x}\mid \mathbf{z})\frac{p_t(\mathbf{x}\mid \mathbf{z})p_{\text{data}}(\mathbf{z}\mid \mathbf{c})}{p_t(\mathbf{x}, \mathbf{c})}\mathrm{d}\mathbf{z} \label{eq:marginal_vf}\\
s_t(\mathbf{x},\mathbf{c}) &= \int s_t(\mathbf{x}\mid \mathbf{z})\frac{p_t(\mathbf{x}\mid \mathbf{z})p_{\text{data}}(\mathbf{z}\mid \mathbf{c})}{p_t(\mathbf{x},\mathbf{c})}\mathrm{d}\mathbf{z} \label{eq:marginal_score}
\end{align}
where the marginal distribution is given by:
\begin{equation}
p_t(\mathbf{x},\mathbf{c})=\int p_t(\mathbf{x}\mid \mathbf{z})p_{\text{data}}(\mathbf{z}\mid \mathbf{c})\mathrm{d}\mathbf{z}
\end{equation}
\end{definition}

\begin{theorem}[Conditional Generation]
\label{thm:conditional_generation}
Consider the stochastic differential equation (SDE):
\begin{align}
\mathbf{X}_0 &\sim p_{\text{init}} \label{eq:sde_init}\\
\mathrm{d}\mathbf{X}_t &= \left[u_t(\mathbf{X}_t,\mathbf{c})+\frac{\sigma_t^2}{2}s_t(\mathbf{X}_t,\mathbf{c})\right]\mathrm{d}t+\sigma_t\mathrm{d}\mathbf{W}_t \label{eq:sde}
\end{align}
where $\mathbf{W}_t$ is a standard Wiener process. Then $\mathbf{X}_t\sim p_t(\cdot \mid \mathbf{c})$ for all $t \in [0,T]$, and consequently:
\begin{equation}
\mathbf{X}_T\sim p_{\text{data}}(\cdot \mid \mathbf{c})
\end{equation}
\end{theorem}

\paragraph{Training Objective.} We consider two complementary objectives.

\emph{(a) Score training (for SDE, $\sigma_t>0$).} We train a score network $s^\theta_t(\mathbf{x},\mathbf{c})$ by denoising score matching:
\begin{align}
\hat{s}_t(\mathbf{x},\mathbf{c}) &= \arg\min_{s^\theta_t} \mathbb{E}_{t, \mathbf{z},\bm{\epsilon}}\left[\left\|s^\theta_t(\mathbf{x}_t,\mathbf{c}) - s_t(\mathbf{x}_t\mid \mathbf{z})\right\|^2\right], \label{eq:score_training}
\end{align}
where $t\sim \mathrm{Unif}(0,T)$, $\mathbf{z}\sim p_{\text{data}}(\cdot\mid\mathbf{c})$, $\bm{\epsilon}\sim p_{\text{init}}$, and $\mathbf{x}_t = \bm{\alpha}_t \odot \mathbf{z} + \bm{\beta}_t \odot \bm{\epsilon}$.
When a score is available, a corresponding velocity can be obtained via
\begin{equation}
\hat{u}_t(\mathbf{x},\mathbf{c}) = \left(\bm{\beta}_t^2 \odot \frac{\dot{\bm{\alpha}}_t}{\bm{\alpha}_t}-\dot{\bm{\beta}}_t \odot \bm{\beta}_t\right) \odot \hat{s}_t(\mathbf{x},\mathbf{c}) + \frac{\dot{\bm{\alpha}}_t}{\bm{\alpha}_t} \odot \mathbf{x}.
\end{equation}

\emph{(b) Velocity training (flow matching, $\sigma_t=0$).} Alternatively, we train a velocity field $u^\theta_t(\mathbf{x},\mathbf{c})$ by regressing to the conditional target from \eqref{eq:conditional_vf}:
\begin{align}
\hat{u}_t(\mathbf{x},\mathbf{c}) &= \arg\min_{u^\theta_t} \mathbb{E}_{t, \mathbf{z},\bm{\epsilon}}\left[\left\|u^\theta_t(\mathbf{x}_t,\mathbf{c}) - u_t(\mathbf{x}_t\mid \mathbf{z})\right\|^2\right], \label{eq:velocity_training}
\end{align}
with the same $(t,\mathbf{z},\bm{\epsilon})$ sampling and $\mathbf{x}_t$ construction as above. 

In our primary setup, we set $\sigma_t=0$ as in \eqref{eq:sigma_zero}; since the trajectory is deterministic and depends only on the drift, we directly train $u_t$ via \eqref{eq:velocity_training} (no score is required).

\paragraph{Specific Time Schedule.} Let $n_s > 0$ be the streaming step-size parameter. We define a lower-triangular schedule as:
\begin{align}
\alpha_t^k &= \mathrm{clamp}(t-\tfrac{k}{n_s},0,1) \label{eq:alpha_schedule}\\
\beta_t^k &= 1 - \alpha_t^k \label{eq:beta_schedule}\\
\sigma_t &= 0 \label{eq:sigma_zero}\\
t &\in [0,1+\frac{K}{n_s}] \label{eq:time_range}
\end{align}

This schedule creates a cascading activation pattern where each frame progressively transitions from noise to data, enabling streaming generation (illustrated in Figure \ref{fig:noise}).

\begin{figure}[t]
    \centering
    \includegraphics[width=\linewidth]{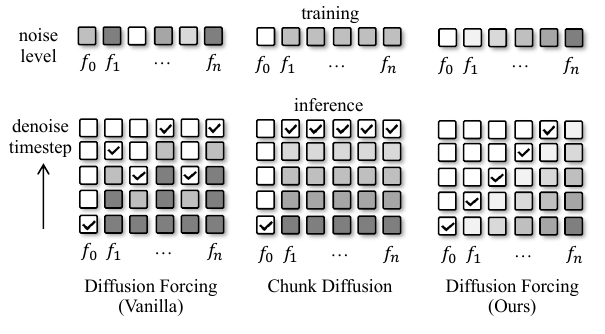}
    \caption{\textbf{Noise Schedule Comparison.} Diffusion forcing samples a random schedule with uncertain active window and mismatches train--test schedule; Chunk diffusion denoises all frames within each chunk uniformly, incurring high response latency. Our triangular schedule denoises only the active window and advances at a constant per-frame rate.}
    \label{fig:noise}
\end{figure}

\subsection{Streaming Training and Inference}
\label{sec:stream}

The vectorized time schedule enables streaming generation, but a naive implementation would still require the full sequence $\mathbf{X}_t^{0:K}$ and $\mathbf{c}^{0:K}$ at each step. We now prove that our time schedule allows frame-by-frame processing with bounded latency.

\begin{definition}[Active Window]
\label{def:active_window}
For time $t \in [0,T]$, define:
\begin{align}
m(t) &= \lceil (t - 1) \cdot n_s \rceil \quad \text{(fully denoised frames)} \label{eq:m_def}\\
n(t) &= \lceil t \cdot n_s \rceil \quad \text{(active frames)} \label{eq:n_def}
\end{align}
\end{definition}

\begin{lemma}[Schedule Saturation]
\label{lem:saturation}
Under the time schedule \eqref{eq:alpha_schedule}--\eqref{eq:beta_schedule}, for any time $t$:
\begin{enumerate}
\item If $k < m(t)$, then $\alpha_t^k = 1$ and $\beta_t^k = 0$
\item If $k \geq n(t)$, then $\alpha_t^k = 0$ and $\beta_t^k = 1$
\end{enumerate}
\end{lemma}

\begin{assumption}[Causal Dependency]\label{assump:causal}
For any $l < K$, the prefix frames depend only on the corresponding prefix of the control signal:
\begin{equation}
p_{\text{data}}(\mathbf{z}^{0:l}\mid \mathbf{c}^{0:K}) = p_{\text{data}}(\mathbf{z}^{0:l}\mid \mathbf{c}^{0:l}).
\end{equation}
\end{assumption}

This assumptions generally holds in streaming settings where future controls do not influence past data.

\begin{theorem}[Streaming Locality]
\label{thm:streaming_locality}
Under \Cref{assump:causal} and the time schedule \eqref{eq:alpha_schedule}--\eqref{eq:beta_schedule}, the velocity field satisfies:
\begin{equation}
u_t(\mathbf{X}_t, \mathbf{c}^{0:K}) = \begin{bmatrix}
\mathbf{0}^{0:m(t)} \\
u_t^{m(t):n(t)}(\mathbf{X}_t^{0:n(t)}, \mathbf{c}^{0:n(t)}) \\
\mathbf{0}^{n(t):K}
\end{bmatrix} \label{eq:streaming_locality}
\end{equation}
where $\mathbf{0}^{i:j}$ denotes a zero vector for indices $i$ to $j$.
\end{theorem}

This theorem implies an efficient implementation: at time $t$, we only compute $u_t$ on the active window $[m(t), n(t))$, requiring only $\mathbf{X}_t^{0:n(t)}$ and $\mathbf{c}^{0:n(t)}$. Frames before $m(t)$ are finalized and frames after $n(t)$ remain pure noise. This yields $1$ frame streaming latency or $\frac{N}{n_s}$ steps and control-response latency is bounded by $n_s$ frames or N steps. We achieve all above while keeping exact likelihood under our schedule (instead of some kind of ELBO), which is different from prior works like diffusion forcing. We provide the algorithm pipeline in \cref{alg:flow_training} and \cref{alg:stream_inference}.

\begin{remark}[Why the triangle matters]
The hard saturation $\alpha_t^k\in\{0,1\}$ outside $[m(t),n(t))$ gives a deterministic, finite cutoff $n_s$ beyond which coordinates are provably pure noise. Without this triangular schedule (\textit{e.g.}, with a random schedule or undeterministic decreasing schedule), there is no fixed finite index after which coordinates are independent of data, so the exact factorization and the simplification in Theorem~\ref{thm:streaming_locality} do not hold.
\end{remark}

\begin{remark}[Bidirectional attention is important] Because the effective context at time $t$ is an interval $[0, n(t))$ (not a strictly causal prefix), a diffusion Transformer instantiated for $u_t$ or $s_t$ should employ bidirectional self-attention within this window. A strictly causal mask would discard the (useful and admissible) future context $k\in[0,n(t))$ and thus be suboptimal under our model assumptions.
\end{remark}

\begin{algorithm}[H]
\footnotesize
\caption{Flood Training}
\label{alg:flow_training}
\begin{algorithmic}[1]
\STATE \textbf{Input:} dataset $\mathcal{D}$ of $(\mathbf{c}^{0:K}, \mathbf{z}^{0:K})$; schedules $\{\alpha_t^k,\beta_t^k\}_{k=0}^{K-1}$; streaming step $n_s$
\FOR{iter $=1$ to $N_{\text{iter}}$}
\STATE sample $(\mathbf{c},\mathbf{z}) \sim \mathcal{D}$; sample $t\sim\mathrm{Unif}(0,T)$; sample $\bm{\epsilon}\sim\mathcal{N}(\mathbf{0},\mathbf{I})$
\STATE $\mathbf{x}_t \leftarrow \bm{\alpha}_t \odot \mathbf{z} + \bm{\beta}_t \odot \bm{\epsilon}$
\STATE $m \leftarrow \lceil (t-1) n_s \rceil$; $n \leftarrow \lceil t n_s \rceil$; $A \leftarrow \{m,\dots,n-1\}$ 
\STATE window $\mathbf{x} \leftarrow \mathbf{x}_t^{0:n}$; window $\mathbf{c} \leftarrow \mathbf{c}^{0:n}$; window $\mathbf{z} \leftarrow \mathbf{z}^{0:n}$
\STATE $\mathbf{u}^{A}_{\text{target}} \leftarrow u_t(\mathbf{x}_t^{0:n} \mid \mathbf{z}^{0:n})[A]$ \COMMENT{Eq.~\eqref{eq:conditional_vf}}
\STATE $\mathbf{u}^{A}_{\text{pred}} \leftarrow u^\theta_t(\mathbf{x}_t^{0:n},\mathbf{c}^{0:n})[A]$
\STATE $L \leftarrow \|\mathbf{u}^{A}_{\text{pred}} - \mathbf{u}^{A}_{\text{target}}\|^2$
\STATE update $\theta \leftarrow \theta - \eta \, \nabla_\theta L$
\ENDFOR
\STATE \textbf{Output:} parameters $\theta$ of velocity field
\end{algorithmic}
\end{algorithm}

\begin{algorithm}[H]
\footnotesize
\caption{Flood Inference ($\sigma_t=0$)}
\label{alg:stream_inference}
\begin{algorithmic}[1]
\STATE \textbf{Input:} control $\mathbf{c}^{0:K}$; trained $u^\theta_t$; schedules $\{\alpha_t^k,\beta_t^k\}_{k=0}^{K-1}$
\STATE initialize $\mathbf{x}_0 \sim \mathcal{N}(\mathbf{0},\mathbf{I})$
\FOR{$t=0$ to $T$ step $\Delta t$}
\STATE $m \leftarrow \lceil (t-1) n_s \rceil$; $n \leftarrow \lceil t n_s \rceil$
\STATE $A \leftarrow \{m,\dots,n-1\}$
\STATE window $\mathbf{x} \leftarrow \mathbf{x}_t^{0:n}$; window $\mathbf{c} \leftarrow \mathbf{c}^{0:n}$
\STATE $\mathbf{v}_A \leftarrow u^\theta_t(\mathbf{x}_t^{0:n},\mathbf{c}^{0:n})[A]$
\STATE $\mathbf{x}_{t+\Delta t}^A \leftarrow \mathbf{x}_t^A + \Delta t \cdot \mathbf{v}_A$
\STATE keep $\mathbf{x}^{0:m}$ fixed; frames $\ge n$ remain noise
\ENDFOR
\STATE \textbf{Output:} $\mathbf{x}_T$ (fully denoised sequence)
\end{algorithmic}
\end{algorithm}

\subsection{Network Architecture}
\label{sec:network}
Our tailored diffusion forcing is instantiated on top of a latent diffusion formulation \citep{rombach2022high}, where we first train a causal variational autoencoder (VAE) to map motion sequences into a compact latent space and then run diffusion only in this latent space. This keeps the streaming latency low and lets the denoiser focus on temporal structure.

\paragraph{Causal VAE}
Different from non-streaming motion works that use VAE or VQ-VAE with bi-directional convolutions \citep{kingma2013auto,van2017neural,guo2024momask}, we adopt a strictly causal design, similar in spirit to the streaming encoder in MotionStreamer \citep{xiao2025motionstreamer}: the decoder at time $t$ does not rely on future frames. Concretely, we start from the causal VAE design used in Wan2.1 for video generation \citep{wan2025wan} and adapt all spatio-temporal blocks to 1D temporal motion sequences. Compared to Wan2.1, our version is simplified: we train it with an $\ell_2$ reconstruction loss and a standard commitment/codebook loss,
so that the encoder produces stable latents for the downstream diffusion task.

\paragraph{DiT with continuous frame-wise text conditioning}
For the latent denoiser, we build on the diffusion transformer (DiT) style backbone \citep{peebles2023scalable} as adopted in Wan2.1 \citep{wan2025wan}, i.e., a single time-embedding pathway shared across blocks, instead of per-block time MLPs. Unlike the original setting for video generation, we use uniform timestep sampling and set the flow-matching time shift to $1$ during training to match our lower-triangular forcing schedule. Text conditioning is applied per frame: we extract token features from a pretrained T5 encoder (max length 128) \citep{raffel2020exploring}, flatten them to 1D, and apply the same rotary embedding to align them with the motion tokens at the current time. Inside attention, each motion frame is only allowed to attend to the text prompt active at that time step via a biased mask. This keeps self-attention over motion clean, while still injecting the latest text, and empirically improves temporal smoothness and text–motion alignment. We provide our pipeline overview in \cref{fig:pipeline}.

\begin{figure*}[t]
    \centering
    \includegraphics[width=\linewidth]{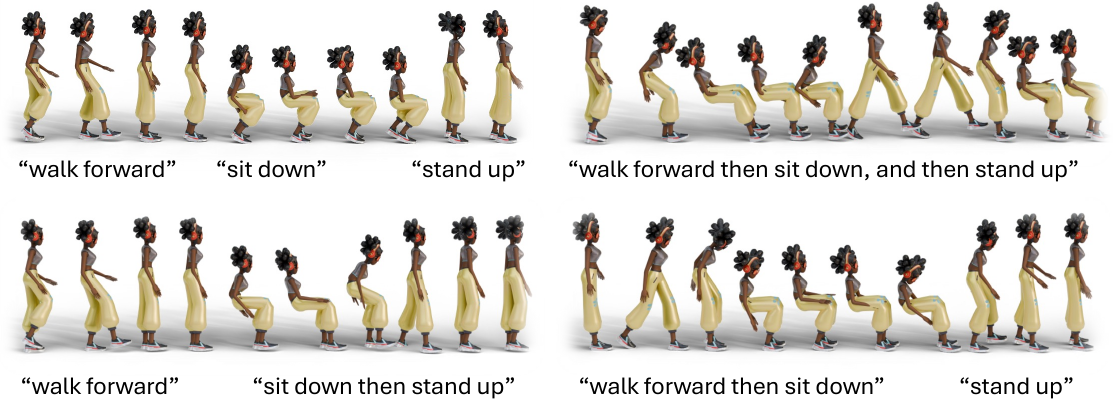}
\caption{\textbf{Comparison of time-varying conditioning.} Our model generates different resulting motions from the same text prompts based on their delivery timing. (Top Left) Prompts are given separately at different frames. (Top Right) All conditions are fed as a single prompt at once. (Bottom Left) Two separate prompts are input early in the sequence. (Bottom Right) The same two separate prompts are input later in the sequence.}
    \label{fig:result1}
\end{figure*}

\begin{figure*}[t]
    \centering
    \includegraphics[width=\linewidth]{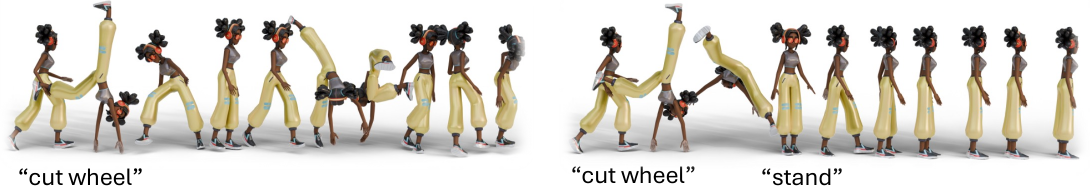}
\caption{\textbf{Comparison of long sequence generation.} (Left) our model will continue to repeat the motion in text prompt if without new prompts come. (Right) in real application, our model could stop current motion by explicitly giving the rest style prompt, such as ``stand''.}
    \label{fig:result2}
\end{figure*}

\section{Experiments}
\label{sec:experiments}

\paragraph{Datasets}
We evaluate on the two standard text-conditioned motion datasets: HumanML3D \citep{guo2022generating} and BABEL \citep{punnakkal2021babel}. 
HumanML3D is built on AMASS and HumanAct12 motions and contains about 14k text–motion pairs (roughly 26 hours), each clip paired with 1–3 natural-language descriptions. 
We use the provided $20$ fps and joint configuration. As in prior work \citep{tevet2022human,guo2024momask}, we adopt the 263-dimensional motion representation distributed with HumanML3D (global velocity, root rotation, joint rotations, foot contact), and use the released text/motion encoders for metrics. 
BABEL \citep{punnakkal2021babel} is an action-centric annotation of AMASS motion, where long motion sequences are segmented and annotated with frame-level action labels (\textit{e.g.}, walking $\rightarrow$ turn $\rightarrow$ jump). 
we follow the official train/val/test split released with these dataset.
Concretely, we follow the official BABEL split and extract continuous multi-action segments to form text–motion pairs, similar to \citep{xiao2025motionstreamer}. This setup matches our streaming assumption, \textit{i.e.}, the text prompt may change inside a single motion sequence.

\paragraph{Metrics}
We follow the evaluation protocol widely used in text-to-motion \citep{guo2022generating,tevet2022human,guo2024momask}. We evaluate Fr{\'e}chet Inception Distance (FID), R-Precision, Diversity, and Multimodality Distance (MM-Dist), Peak Jerk (PJ) and Area Under the Jerk (AUJ).
For FID, we adopt the same motion encoderas HumanML3D \citep{guo2022generating} for encoding the real motions and generated motions.
For R-Precision (top-$k$), given a generated motion, we encode it and retrieve among ground-truth text description plus mismatched texts. If the ground-truth text is within the top-$k$ ranked texts, it is counted as correct. We report R-Precision@1/2/3, following \citet{guo2022generating,tevet2022human}.
For Diversity, we measure whether the model collapses by randomly sampling several generated motions, and computing the average pairwise $\ell_2$ distance.
For Multimodality Distance (MM-Dist), we measure the distance between motions and texts. For PJ \citep{barquero2024seamless}, we calculate the maximum value throughout the transition motion over all joints, and similarly, AUJ\citep{barquero2024seamless} is for the area under the jerk curve.

\begin{table*}[thb]
\centering
\scalebox{0.85}{
\begin{tabular}{l c c c c c c c c c}
&  & \multicolumn{6}{c}{\textbf{HumanML3D}} & \multicolumn{2}{c}{\textbf{BABEL}}\\
\cmidrule(lr){3-8}\cmidrule(lr){9-10}
 & stream & R@1$\uparrow$ & R@2$\uparrow$ & R@3$\uparrow$ & FID$\downarrow$ & MM-Dist$\downarrow$ & Diversity$\rightarrow$ & PJ$\rightarrow$ & AUJ$\downarrow$\\
\midrule
Real motion & ~ & \et{0.511}{.003} & \et{0.703}{.003} & \et{0.797}{.002} & \et{0.002}{.000} & \et{2.974}{.008} & \et{9.503}{.065} & \et{1.100}{.000} & \et{41.20}{.000}\\
VAE reconstruction & ~ & - & - & - & \et{0.029}{.000} & - & \et{9.462}{.031} & - & - \\
TM2T \citep{guo2022tm2t} & ~ & \et{0.424}{.003} & \et{0.618}{.003} & \et{0.729}{.002} & \et{1.501}{.017} & \et{3.467}{.011} & \et{8.589}{.076} & - & - \\
T2M \citep{guo2022generating} ~ & ~ & \et{0.455}{.003} & \et{0.636}{.003} & \et{0.736}{.002} & \et{1.087}{.021} & \et{3.347}{.008} & \et{9.175}{.083} & - & - \\
MDM \citep{tevet2022human} & ~ & - & - & \et{0.611}{.007} & \et{0.544}{.044} & \et{5.566}{.027} & \ets{9.559}{.086} & - & - \\
MLD \citep{chen2023executing} & ~ & \et{0.481}{.003} & \et{0.673}{.003} & \et{0.772}{.002} & \et{0.473}{.013} & \et{3.196}{.010} & \et{9.724}{.082} & - & - \\
MotionDiffuse \citep{zhang2024motiondiffuse} & ~ & \et{0.491}{.001} & \et{0.681}{.001} & \et{0.782}{.001} & \et{0.630}{.001} & \et{3.113}{.001} & \et{9.410}{.049} & - & - \\
T2M-GPT \citep{zhang2023generating} & ~ & \et{0.492}{.003} & \et{0.679}{.002} & \et{0.775}{.002} & \et{0.141}{.005} & \et{3.121}{.009} & \et{9.722}{.082} & --& - \\
ReMoDiffuse \citep{zhang2023remodiffuse} & ~ & \et{0.510}{.005} & \et{0.698}{.006} & \et{0.795}{.004} & \et{0.103}{.004} & \et{2.974}{.016} & \et{9.018}{.075} & - & - \\
MoMask \citep{guo2024momask} & ~ & \ets{0.521}{.002} & \ets{0.713}{.002} & \ets{0.807}{.002} & \etb{0.045}{.002} & \et{2.958}{.008} & \et{9.677}{.032} & - & - \\
\cline{1-10}
PRIMAL \citep{zhang2025primal} & \cmark & \et{0.497}{.003} & \et{0.681}{.003} & \et{0.780}{.002} & \et{0.511}{.006} & \et{3.120}{.019} & \etb{9.520}{.068} & \et{1.304}{.105} & \et{19.36}{.969} \\
MotionStreamer \citep{xiao2025motionstreamer} & \cmark & \et{0.513}{.003} & \et{0.705}{.006} & \et{0.802}{.003} & \et{0.092}{.003} & \ets{2.909}{.015} & \et{9.722}{.037} & \ets{0.912}{.049} & \ets{16.57}{.762} \\
\textbf{FloodDiffusion} & \cmark & \etb{0.523}{.002} & \etb{0.717}{.002} & \etb{0.810}{.003} & \ets{0.057}{.002} & \etb{2.887}{.007} & \et{9.579}{.062} & \etb{0.713}{.039} & \etb{14.05}{.663}  \\
\end{tabular}
}
\caption{\textbf{Quantitative evaluation on HumanML3D and BABEL test sets.} We report alignment (R@k$\uparrow$), quality (FID$\downarrow$), multimodality (MM-Dist$\downarrow$), and streaming quality (PJ$\rightarrow$, AUJ$\downarrow$). $\rightarrow$ means closer to `Real motion' is better; $\pm$ indicates 95\% confidence intervals; `–' means not applicable. We compare against SOTA non-streaming (\texttt{MoMask} \citep{guo2024momask}), etc., and streaming (\texttt{PRIMAL} \citep{zhang2025primal}, \texttt{MotionStreamer} \citep{xiao2025motionstreamer}) methods. FloodDiffusion achieves the best R@k and MM-Dist, a competitive FID (0.057) on HumanML3D, and outperforms all streaming baselines on BABEL.}
\label{tab:quantitative_eval}
\end{table*}

\begin{table}[t]
\centering
\scalebox{0.83}{
\setlength{\tabcolsep}{6pt}
\begin{tabular}{l c c c}
& Preference$\uparrow$ & Transition$\uparrow$ & Consistency$\uparrow$ \\
\midrule
Real motion                & 0.224 & 0.299 & 0.280 \\
MotionStreamer \citep{xiao2025motionstreamer}    & -0.338 & -0.136 & -0.055 \\
PRIMAL \citep{zhang2025primal}            & -0.599 & -0.315 & -0.203 \\
FloodDiffusion     & \textbf{0.024} & \textbf{0.152} & \textbf{-0.021} \\
\end{tabular}
}
\caption{\textbf{Bradley--Terry user study with 100 participants.}
Three generative models (PRIMAL, MotionStreamer, FloodDiffusion)  
are compared against ground-truth motion across three perceptual metrics.}
\label{tab:bt_results}
\end{table}

\begin{table}[t]
\centering
\scalebox{0.90}{
\setlength{\tabcolsep}{6pt}
\begin{tabular}{l c c c c}
& FID$\downarrow$ & R@3$\uparrow$ & MM-Dist$\downarrow$ & Div$\rightarrow$ \\
\midrule
Real motion         & 0.002 & 0.797 & 2.974 & 9.503 \\
w/o bi-direction    & 3.377 & 0.625 & 4.296  & 7.942  \\
w/o low triangle    & 3.883 & 0.532 & 4.651  & 8.497  \\
Ours    & 0.057 & 0.810 & 2.887  & 9.632  \\
\end{tabular}
}
\caption{\textbf{Ablation on core design choices.} We evaluate the impact of removing bi-directional attention (`w/o bi-direction') and the lower-triangular time scheduler (`w/o low triangle'). Results show that removing either results in a significant performance collapse across all metrics, especially FID (from 0.057 to 3.377 and 3.883, respectively).}
\label{tab:abl_singlecol}
\end{table}

\begin{table}[t]
\centering
\scalebox{0.80}{
\setlength{\tabcolsep}{6pt}
\begin{tabular}{l c c c c}
& ACCEL$\downarrow$ & MPJPE$\uparrow$ & PAMPJPE$\downarrow$ & FID$\downarrow$ \\
\midrule
VQ-VAE \citep{zhang2023generating}    & 0.0062 & 0.0545 & 0.0341  & 0.1259  \\
CausalVAE \citep{xiao2025motionstreamer}    & 0.0103 & 0.0576 & 0.0351  & 0.0271  \\
Wan CausalVAE \citep{wan2025wan} & 0.0094 & 0.0555 & 0.0333  & 0.0290  \\
\end{tabular}
}
\caption{\textbf{Ablation on Causal VAE architectures.} We compare VAE reconstruction quality using VQ-VAE \citep{zhang2023generating}, MotionStreamer's CausalVAE \citep{xiao2025motionstreamer}, and our adopted Wan CausalVAE \citep{wan2025wan}. The results show comparable performance across the different causal variants.}
\label{tab:abl_singlecol2}
\end{table}

\paragraph{Training Settings}
For our Causal VAE, we use a temporal downsampling factor of $4$ and set the latent channel dimension to $4$. We train it with AdamW (learning rate $2\times10^{-4}$) using a constant schedule with a warm-up of $T_{\text{warm}}=1000$, for a total of $300\text{K}$ steps. The diffusion backbone uses a streaming slope size $n_s=5$. It is optimized with AdamW (learning rate $1\times10^{-3}$) and a cosine annealing schedule with $T_{\max}=1000$, trained for $50\text{K}$ steps. Both models are trained on H200 GPUs.

We first train on HumanML3D alone to evaluate the standard text-to-motion setting, where every frame in a clip shares the same text description, following \citep{guo2022generating,tevet2022human}. Then we train a joint model on HumanML3D and BABEL \citep{punnakkal2021babel} to evaluate scenarios with time-varying prompts; in this setting, each frame (or short segment) is conditioned on its corresponding text, requiring the model to update the motion immediately as the text changes. In both training settings, the model operates in streaming mode, denoising the current buffer and enabling frame-by-frame output without waiting for the full sequence.

\subsection{Comparison to state-of-the-art approaches}
Our evaluations are divided into two parts: non-streaming evaluation on HumanML3D and streaming evaluation on BABEL. For each dataset, training and testing are performed on the standard splits.

\paragraph{Baselines}
We first train our model on HumanML3D for non-streaming comparison, where the input is a single text and the output is a complete motion matching the ground-truth length. We compare our method with previous state-of-the-art models up to MoMask\citep{guo2024momask}, including the retrieval-based method ReMoDiffuse\citep{zhang2023remodiffuse}, the auto-regressive (AR) method T2M-GPT \citep{zhang2023generating}, and the diffusion-based method MDM \citep{tevet2022human}.
Then, we train our model on jointly on HumanML3D and the BABEL dataset from scratch for streaming comparison. Following previous works \citep{zhang2025primal, xiao2025motionstreamer}, the evaluation inputs past motion, past text, and current text to generate the subsequent motion. We reproduce PRIMAL \citep{zhang2025primal} and MotionStreamer \citep{xiao2025motionstreamer} using their official code, adapting them to the 263D \citep{guo2022generating} HumanML3D motion representation to ensure a consistent comparison with our non-streaming results. We note that all frameworks could be trained on other motion representations, such as the 272D representation in MotionStreamer \citep{xiao2025motionstreamer} or the 330D SMPLX \citep{pavlakos2019expressive} 6D-rotation \citep{zhou2019continuity} representation.

\paragraph{Objective Evaluations}
As shown in Table \ref{tab:quantitative_eval}, our method performs better on most non-streaming metrics and leads on all streaming metrics. For the non-streaming evaluation, our method outperforms existing baselines in R-Precision and Multimodal Distance, and performs on par with MoMask\citep{guo2024momask} in FID. For streaming-specific metrics, our method achieves a PJ closest to real motion and performs better on AUJ. Overall, the results demonstrate our method can generate streaming motion at a state-of-the-art quality.

\paragraph{Subjective Evaluations}
We conduct a user study with 100 participants comparing three generative models
(PRIMAL, MotionStreamer, and our FloodDiffusion) against ground-truth motion.
Across all three evaluation metrics (Preference, Transition, and Consistency),
FloodDiffusion attains the highest Bradley--Terry score among the
generative baselines (Table~\ref{tab:bt_results}). Besides, we show subjective samples in Figure \ref{fig:teaser}, \ref{fig:result1} and \ref{fig:result2}.  

\subsection{Ablation Study}
\paragraph{Bi-directional attention}
We first evaluate the importance of bi-directional attention, as shown in Table \ref{tab:abl_singlecol}. Unlike full-chunk diffusion models \citep{zhang2025primal}, where switching from bi-directional to causal attention partially influences the results (\textit{e.g.}, FID increasing from 0.51 to 0.92), we found bi-directional attention is critical for diffusion forcing. In our framework, using causal attention causes the FID to degrade significantly, from 0.054 to 3.377. This suggests that the denoising process requires the full context of the active window, especially when frames within that window are at different timesteps.

\paragraph{Lower-triangular timestep sampling scheduler}
We then compare the random timestep sampling scheduler used in vanilla diffusion forcing with our deterministic lower-triangular scheduler. Both experiments are conducted using bi-directional attention. As shown in Table \ref{tab:abl_singlecol}, we found that the random scheduler results in a significantly higher FID (3.883) compared to our lower-triangular scheduler (0.057) after 1M training iterations. This indicates that the structured, cascading noise schedule is highly beneficial for this task.

\begin{figure}[t]
\centering
\includegraphics[width=\columnwidth]{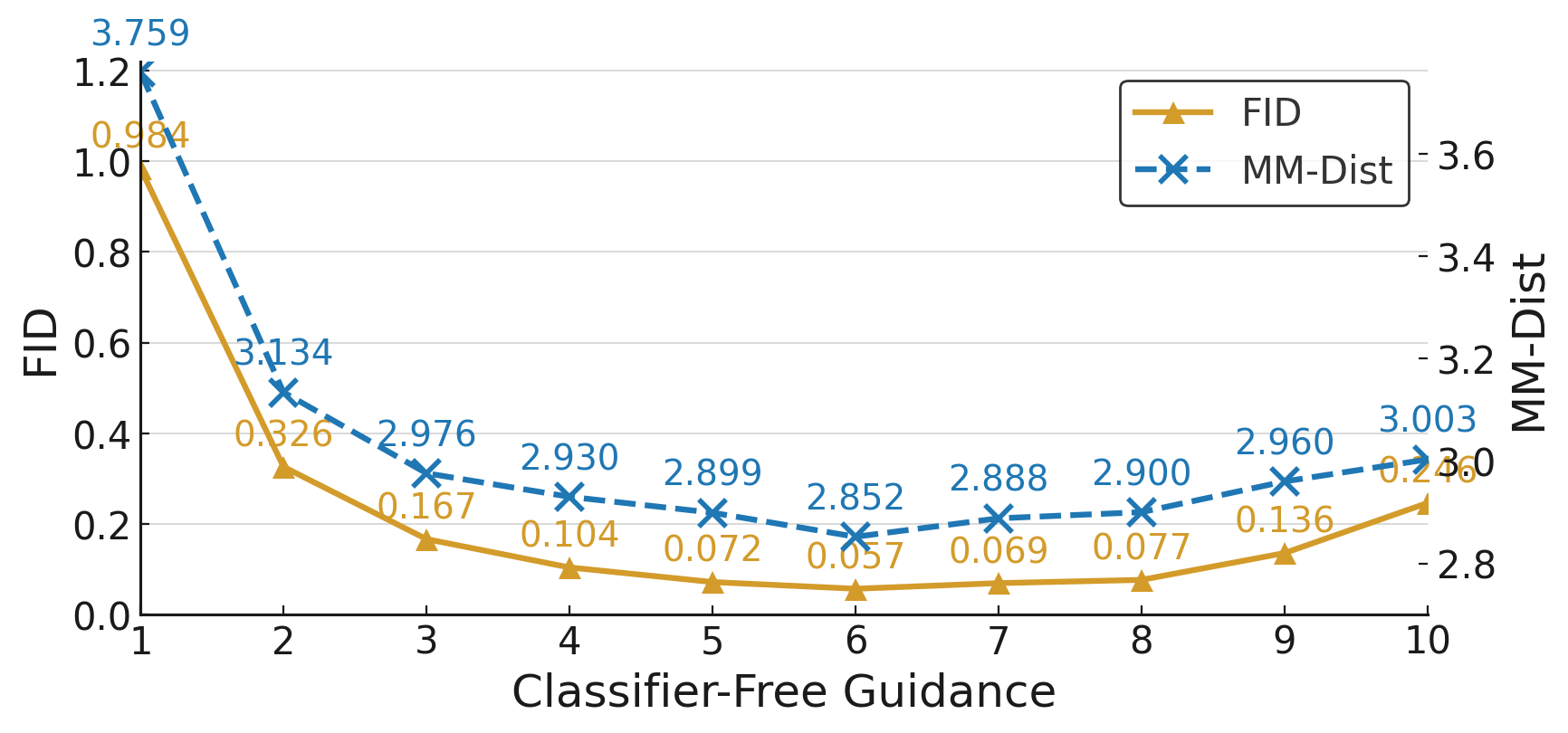}
\caption{\textbf{Effect of Classifier-Free Guidance (CFG) scale.} We report FID (left axis, $\downarrow$) and MM-Dist (right axis, $\downarrow$) with the CFG scale. Both metrics improve significantly as the scale increases from 1, achieving an optimal trade-off at CFG=6 (FID=0.057, MM-Dist=2.852). }
\label{fig:cfg_ablation}
\end{figure}

\paragraph{Classifier-free guidance}
We report the performance of our model under different classifier-free guidance (CFG) settings. Our best results in this paper are obtained with a text CFG scale of 6. As shown in the figure \ref{fig:cfg_ablation}, and similar to previous work, the model performance (FID and Multimodal Distance) degrades when CFG is not used.

\paragraph{Choice of Causal VAE Architecture}
We report the performance differences for various Causal VAE architectural choices. As shown in Table \ref{tab:abl_singlecol2}, the performance of different Causal VAE structures is similar when the latent dimension is the same. The results in the table are from a latent dimension of 4. Therefore, we keep the Wan-based Causal VAE to maintain architectural similarity with the video generation community.   

\subsection{Discussion}
\paragraph{Advantage of diffusion forcing for streaming systems}
We discuss the differences between our approach and other streaming architectures.
(i) Training-inference mismatch. Chunk-based diffusion and AR models with diffusion heads are often trained to complete full motions, while streaming is achieved via inference-time optimizations, such as stopping the current generation and refreshing with new conditions. In contrast, our method is trained directly on time-varying text conditions and does not require manual detection of new prompts, stopping, or refreshing.
(ii) Latency and window size trade-off. Similar to AR-based methods, diffusion forcing can leverage a long context window (\textit{e.g.}, 100 frames of history) while generating new motion with low latency (\textit{e.g.}, 5 new frames). This avoids the high "first-token" latency seen in chunk-based diffusion, which must generate a full chunk before outputting.

\paragraph{Limitation}
Our system does not include instruction finetuning or explicit semantic memory of past generated motion. This makes it challenging to follow text prompts that require understanding past motion, such as ``repeat your last action''. One potential solution is to use an auxiliary LLM system to interpret these text prompts into specific commands. Also, due to the lack of long-term stylized data, the model was not trained or objectively evaluated on style shifting or long-term consistency issues. A potential solution is obtaining pseudo-labeled long-motion data from video or games for training.

\section{Conclusions}
\label{sec:conclusions}
In this paper, we introduce a new streaming human motion generation framework based on diffusion forcing. We improved the vanilla diffusion forcing on both the attention architecture and the training timestep sampling scheduler. We found that with these modifications, a diffusion forcing framework can work effectively on this task and achieve state-of-the-art performance. We also provided proof that our modifications guarantee the network correctly models the output distribution. Since the proposed framework can generally accept different time-varying conditions, future work includes introducing and fusing more conditions such as audio, force, and environment feedbacks.

{
    \small
    \bibliographystyle{ieeenat_fullname}
    \bibliography{main}
}

\newpage
\input{appendix}

\end{document}

%% file: appendix.tex
\clearpage
\setcounter{page}{1}
\maketitlesupplementary
\appendix

\noindent This supplemental document contains four sections: 
\begin{itemize}[leftmargin=*]

\item Video and Codes (Section \ref{sec:sup1}). 

\item Detailed Proofs of Theoretical Results (Section \ref{sec:sup2}).

\item Baseline Implementation Details (Section \ref{sec:sup3}).

\item Hyper-parameter Search (Section \ref{sec:sup4}).

\item Details of User Study (Section \ref{sec:sup5}).

\end{itemize}

\section{Video and Codes}
\label{sec:sup1}

We provide a project page that contains separate videos to demonstrate the performance of our system, including:

\begin{itemize}
    \item Results for streaming human motion generation.
    \item Influence on the position to giving the text prompt. 
    \item Stop the motion via neural language command.
    \item Compare with non-streaming methods. 
    \item Compare with other streaming methods.
    \item Ablation study for the bi-directional attention and time scheduler. 
\end{itemize}

\section{Detailed Proofs of Theoretical Results}
\label{sec:sup2}

\begin{proposition}[Vectorized Conditional Dynamics]
\label{prop:conditional_dynamics_appendix}
This proposition restates Proposition~\ref{prop:conditional_dynamics} in the main text. Under the vectorized time schedule, the conditional vector field and score function of the Gaussian path
\begin{equation}
    p_t(\mathbf{x}\mid \mathbf{z}) = \prod_{k=0}^{K-1} \mathcal{N}\big(x^k; \alpha_t^k z^k, (\beta_t^k)^2\mathbf{I}\big)
\end{equation}
are given by
\begin{align}
u_t(\mathbf{x}\mid \mathbf{z}) &= \left(\dot{\bm{\alpha}}_t - \frac{\dot{\bm{\beta}}_t}{\bm{\beta}_t}\odot \bm{\alpha}_t\right) \odot \mathbf{z} + \left(\frac{\dot{\bm{\beta}}_t}{\bm{\beta}_t}\right) \odot \mathbf{x}\\
s_t(\mathbf{x}\mid \mathbf{z}) &= -\frac{(\mathbf{x}-\bm{\alpha}_t \odot \mathbf{z})}{\bm{\beta}_t^2}
\end{align}
\end{proposition}

\begin{proof}
Since each dimension is independent in the Gaussian distribution, we have
\begin{equation}
    p_t(\bx\mid \bz)
    = \mathcal{N}\!\big(\bx;\, \bm{\alpha}_t \odot \bz, \operatorname{diag}(\bm{\beta}_t^2)\big)
    = \prod_{i=1}^K p_t^i(x^i\mid z^i)
\end{equation}
where $p_t^i(x^i\mid z^i)$ denotes the marginal density of the $i$-th coordinate.

For a Gaussian $\mathcal{N}(x; \mu, \Sigma)$, the score is
\begin{equation}
    \nabla_x \log p(x) = -\Sigma^{-1}(x-\mu)
\end{equation}
Applying this to our diagonal-covariance case gives
\begin{align}
    s_t(\bx\mid \bz) &= \nabla_{\bx} \log p_t(\bx\mid \bz) \\
    &= -\operatorname{diag}(\bm{\beta}_t^{-2})\big(\bx - \bm{\alpha}_t \odot \bz\big) \\[-0.25em]
    &= -\frac{\bx-\bm{\alpha}_t \odot \bz}{\bm{\beta}_t^2}
\end{align}
which matches the claimed expression.

We now verify that $u_t$ defines a valid probability flow field by checking the continuity equation
\begin{equation}
    \partial_t p_t(\bx) + \nabla_{\bx} \cdot \big(p_t(\bx)\, u_t(\bx)\big) = 0
\end{equation}
Using the factorization $p_t(\bx) = \prod_{i=1}^K p_t^i(x^i)$ and writing $u_t(\bx) = (u_t^1(x^1),\dots,u_t^K(x^K))$, we obtain
\begin{align}
    \partial_t p_t(\bx)
    &= \partial_t \Big(\prod_{i=1}^K p_t^i(x^i)\Big) \\
    &= \sum_{i=1}^K \Bigg[\Big(\partial_t p_t^i(x^i)\Big)\prod_{j\neq i} p_t^j(x^j)\Bigg]
\end{align}
\begin{align}
    &\nabla_{\bx} \cdot \big(p_t(\bx)\, u_t(\bx)\big) \\
    &= \sum_{i=1}^K \partial_{x^i}\Big(p_t(\bx)\, u_t^i(x^i)\Big) \\
    &= \sum_{i=1}^K \partial_{x^i}\Bigg(\Big[\prod_{j=1}^K p_t^j(x^j)\Big] u_t^i(x^i)\Bigg) \\
    &= \sum_{i=1}^K \Bigg[\Big(\partial_{x^i} p_t^i(x^i)\Big) u_t^i(x^i)\Bigg]\prod_{j\neq i} p_t^j(x^j) \\
    &\quad + \sum_{i=1}^K \Bigg[p_t^i(x^i)\, \partial_{x^i} u_t^i(x^i)\Bigg]\prod_{j\neq i} p_t^j(x^j)
\end{align}
Therefore,
\begin{align}
    &\partial_t p_t(\bx) + \nabla_{\bx} \cdot \big(p_t(\bx)\, u_t(\bx)\big) \\
    &= \sum_{i=1}^K \Bigg[\prod_{j\neq i} p_t^j(x^j)\Bigg]
      \Big(\partial_t p_t^i(x^i) + \partial_{x^i}\big(p_t^i(x^i)\, u_t^i(x^i)\big)\Big)
\end{align}
The term in parentheses is exactly the one-dimensional continuity equation for the $i$-th coordinate, which holds by construction of the scalar Gaussian flow. 

Hence each summand is zero, and the whole sum vanishes, proving that $u_t$ is a valid vector field for the factorized Gaussian path.
\end{proof}

\begin{theorem}[Conditional Generation]
\label{thm:conditional_generation_appendix}
This theorem restates Theorem~\ref{thm:conditional_generation} in the main text. Consider the SDE
\begin{align}
\mathbf{X}_0 &\sim p_{\text{init}},\\
\mathrm{d}\mathbf{X}_t &= \left[u_t(\mathbf{X}_t,\mathbf{c})+\frac{\bm{\sigma}_t^2}{2}s_t(\mathbf{X}_t,\mathbf{c})\right]\mathrm{d}t+\sigma_t\mathrm{d}\mathbf{W}_t
\end{align}
where $u_t$ and $s_t$ are respectively the marginal vector field and score of the path $p_t(\mathbf{x},\mathbf{c})$. Then the marginal law of $\mathbf{X}_t$ satisfies $\mathbf{X}_t \sim p_t(\cdot \mid \mathbf{c})$ for all $t$, and in particular $\mathbf{X}_T \sim p_{\text{data}}(\cdot \mid \mathbf{c})$.
\end{theorem}

\begin{proof}
We verify the claim by appealing to the Fokker--Planck equation. We first work in the general matrix-valued setting, then specialize to the diagonal case.

\paragraph{General case.}
Consider a generic SDE of the form
\begin{equation}
    \mathrm{d}\bx_t = u^{fp}_t(\bx_t)\,\mathrm{d}t + \sigma_t\,\mathrm{d}\mathbf{W}_t
\end{equation}
with initial condition $\bx_0 \sim p_{\text{init}}$, where $\sigma_t$ is a diffusion matrix (so that $\sigma_t \sigma_t^\top$ is the covariance). The Fokker--Planck equation states that the marginal density $p_t$ evolves as
\begin{equation}
        \partial_t p_t(\bx)
        = -\nabla_{\bx}\cdot\big(p_t(\bx) u^{fp}_t(\bx)\big)
            + \frac{1}{2}\nabla_{\bx}\cdot\Big(\sigma_t \sigma_t^\top \, \nabla_{\bx} p_t(\bx)\Big)
\end{equation}

To match the desired marginal $p_t(\cdot\mid\mathbf{c})$, we decompose the drift by adding and subtracting a score-correction term (we omit $\mathbf{c}$ for brevity):
\begin{align}
    u_t(\bx) &= u_t(\bx) + \frac{\sigma_t \sigma_t^\top}{2} s_t(\bx) - \frac{\sigma_t \sigma_t^\top}{2} s_t(\bx)
\end{align}
Substituting this into the FP equation yields
\begin{align}
    \partial_t p_t(\bx) &= -\nabla_{\bx}\cdot\big(p_t u_t\big) \\
    &= -\nabla_{\bx}\cdot\Big(p_t\big(u_t + \tfrac{\sigma_t \sigma_t^\top}{2} s_t\big)\Big) \\
       &\quad + \nabla_{\bx}\cdot\Big(p_t \, \tfrac{\sigma_t \sigma_t^\top}{2} s_t\Big)
\end{align}

Applying the identity $s_t = \nabla_{\bx} \log p_t$, which gives $p_t \nabla_{\bx} \log p_t = \nabla_{\bx} p_t$, we find
\begin{equation}
    \nabla_{\bx}\cdot\Big(p_t \, \tfrac{\sigma_t \sigma_t^\top}{2} s_t\Big)
    = \frac{1}{2}\nabla_{\bx}\cdot\Big(\sigma_t \sigma_t^\top \, \nabla_{\bx} p_t\Big)
\end{equation}
This precisely reproduces diffusion term in the FP equation, leaving
\begin{align}
    \partial_t p_t(\bx)
    &= -\nabla_{\bx}\cdot\Big(p_t(\bx)\big(u_t(\bx) + \tfrac{\sigma_t \sigma_t^\top}{2} s_t(\bx)\big)\Big) \\
    &\quad + \frac{1}{2}\nabla_{\bx}\cdot\Big(\sigma_t \sigma_t^\top \, \nabla_{\bx} p_t(\bx)\Big)
\end{align}

\paragraph{Specialization to diagonal diffusion.}
We now restrict to the diagonal case where $\sigma_t = \operatorname{diag}(\bm{\sigma}_t)$ and hence $\sigma_t \sigma_t^\top = \operatorname{diag}(\bm{\sigma}_t^2)$, with $\bm{\sigma}_t^2$ denoting the element-wise square of the noise standard-deviation vector. In this case, the drift correction simplifies to
\begin{equation}
    u_t(\bx,\mathbf{c}) + \tfrac{\sigma_t \sigma_t^\top}{2} s_t(\bx,\mathbf{c})
    = u_t(\bx,\mathbf{c}) + \tfrac{\bm{\sigma}_t^2}{2} s_t(\bx,\mathbf{c})
\end{equation}
where all operations are now coordinate-wise. 

Therefore, for any diagonal diffusion covariance $\Sigma_t = \operatorname{diag}(\bm{\sigma}_t^2)$, the SDE~\eqref{eq:sde} with drift $u^{fp}_t = u_t + \tfrac{\bm{\sigma}_t^2}{2} s_t$ produces the target marginal distribution. Since $u_t$ and $s_t$ are defined (Definition~\ref{def:marginal_dynamics}) to generate the marginal path $p_t(\cdot\mid \mathbf{c})$, we conclude that $\bx_t \sim p_t(\cdot\mid \mathbf{c})$ for all $t \in [0,T]$, and in particular $\bx_T \sim p_{\text{data}}(\cdot\mid \mathbf{c})$.
\end{proof}

\begin{lemma}[Schedule Saturation]
\label{lem:saturation_appendix}
This lemma restates Lemma~\ref{lem:saturation} in the main text. Under the schedule
\begin{align}
\alpha_t^k &= \mathrm{clamp}(t-\tfrac{k}{n_s},0,1), & \beta_t^k &= 1-\alpha_t^k
\end{align}
for any $t$ we have
\begin{enumerate}
    \item If $k < m(t) = \lceil (t-1)n_s \rceil$, then $\alpha_t^k = 1$ and $\beta_t^k = 0$;
    \item If $k \geq n(t) = \lceil t n_s \rceil$, then $\alpha_t^k = 0$ and $\beta_t^k = 1$.
\end{enumerate}
\end{lemma}

\begin{proof}
Fix $t$ and $k$. By definition
\begin{equation}
    \alpha_t^k = \mathrm{clamp}\big(t-\tfrac{k}{n_s},0,1\big)
\end{equation}
If $k < m(t) = \lceil (t-1)n_s \rceil$, then $k / n_s < t-1$, hence $t - k/n_s > 1$ and $\alpha_t^k = 1$. Consequently $\beta_t^k = 1-\alpha_t^k = 0$. 

Conversely, if $k \geq n(t) = \lceil t n_s \rceil$, then $k / n_s \geq t$ so that $t-k/n_s \leq 0$ and thus $\alpha_t^k = 0$ and $\beta_t^k = 1$. This establishes both claims.
\end{proof}

\begin{theorem}[Streaming Locality]
\label{thm:streaming_locality_appendix}
This theorem restates Theorem~\ref{thm:streaming_locality} in the main text. Under Assumption~\ref{assump:causal} and the triangular schedule \eqref{eq:alpha_schedule}--\eqref{eq:beta_schedule}, the velocity field factorizes as
\begin{equation}
u_t(\mathbf{X}_t, \mathbf{c}^{0:K}) = \begin{bmatrix}
\mathbf{0}^{0:m(t)} \\
u_t^{m(t):n(t)}(\mathbf{X}_t^{0:n(t)}, \mathbf{c}^{0:n(t)}) \\
\mathbf{0}^{n(t):K}
\end{bmatrix}
\end{equation}
\end{theorem}

\begin{proof}
We first observe that the conditional velocity $u_t$, noise $\epsilon_t$, data prediction $z_t$, and score $s_t$ can all be expressed as affine combinations of $\mathbf{z}$ and $\mathbf{x}$:

\begin{align}
u_t(\mathbf{x}\mid \mathbf{z}) &= \left(\dot{\bm{\alpha}}_t - \frac{\dot{\bm{\beta}}_t}{\bm{\beta}_t}\odot \bm{\alpha}_t\right) \odot \mathbf{z} + \left(\frac{\dot{\bm{\beta}}_t}{\bm{\beta}_t}\right) \odot \mathbf{x}\\
\epsilon_t(\mathbf{x}\mid \mathbf{z}) &= \left(-\frac{\bm{\alpha}_t}{\bm{\beta}_t}\right)\odot \mathbf{z} + \left(-\frac{1}{\bm{\beta}_t}\right)\odot \mathbf{x}\\
z_t(\mathbf{x}\mid \mathbf{z}) &= 1 \odot \mathbf{z} + 0 \odot \mathbf{x}\\
s_t(\mathbf{x}\mid \mathbf{z}) &= \left(\frac{\bm{\alpha}_t}{\bm{\beta}_t^2}\right)\odot \mathbf{z} + \left(-\frac{1}{\bm{\beta}_t^2}\right)\odot \mathbf{x}
\end{align}

These prediction targets share a common form:
\begin{equation}
f_t(\mathbf{x}\mid \mathbf{z}) = \mathbf{a}_t \odot \mathbf{z} + \mathbf{b}_t \odot \mathbf{x}
\end{equation}
where $\odot$ denotes element-wise multiplication. The corresponding marginal prediction conditioned on $\mathbf{c}$ is given by the posterior expectation:
\begin{align}
f_t(\mathbf{x},\mathbf{c}) &= \int f_t(\mathbf{x}\mid \mathbf{z})\frac{p_t(\mathbf{x}\mid \mathbf{z})p_{\text{data}}(\mathbf{z}\mid \mathbf{c})}{p_t(\mathbf{x}, \mathbf{c})}\mathrm{d}\mathbf{z} \\
&= \mathbf{a}_t \odot \int \mathbf{z}\frac{p_t(\mathbf{x}\mid \mathbf{z})p_{\text{data}}(\mathbf{z}\mid \mathbf{c})}{p_t(\mathbf{x}, \mathbf{c})}\mathrm{d}\mathbf{z} + \mathbf{b}_t \odot \mathbf{x} \\
&= \mathbf{a}_t \odot g_t(\mathbf{x},\mathbf{c}) + \mathbf{b}_t \odot \mathbf{x}
\end{align}

Thus, the core task is to compute the posterior mean of the data $\mathbf{z}$, denoted as $g_t(\mathbf{x},\mathbf{c}) \coloneq \mathbb{E}[\mathbf{z} \mid \mathbf{x}, \mathbf{c}]$.

Under our specialized triangular schedule, we partition the sequence into three regions: finalized ($0:m$), active ($m:n$), and future ($n:K$). We denote the corresponding sub-vectors as $\mathbf{v}^1, \mathbf{v}^2, \mathbf{v}^3$ for any vector $\mathbf{v}$:
\begin{align}
\mathbf{z} &= [\mathbf{z}^{0:m},\mathbf{z}^{m:n},\mathbf{z}^{n:K}] = [\mathbf{z}^\mathbf{1},\mathbf{z}^\mathbf{2},\mathbf{z}^\mathbf{3}]\\
\mathbf{x} &= [\mathbf{x}^{0:m},\mathbf{x}^{m:n},\mathbf{x}^{n:K}] = [\mathbf{x}^\mathbf{1},\mathbf{x}^\mathbf{2},\mathbf{x}^\mathbf{3}]\\
\mathbf{c} &= [\mathbf{c}^{0:m},\mathbf{c}^{m:n},\mathbf{c}^{n:K}] = [\mathbf{c}^\mathbf{1},\mathbf{c}^\mathbf{2},\mathbf{c}^\mathbf{3}]
\end{align}

The conditional distribution $p_t(\mathbf{x}\mid \mathbf{z})$ factorizes as:
\begin{align}
p_t(\mathbf{x}\mid \mathbf{z}) &= p_t(\mathbf{x}^\mathbf{1} \mid \mathbf{z}^\mathbf{1})p_t(\mathbf{x}^\mathbf{2} \mid \mathbf{z}^\mathbf{2})p_t(\mathbf{x}^\mathbf{3} \mid \mathbf{z}^\mathbf{3}) \\
&= \delta(\mathbf{x}^\mathbf{1} - \mathbf{z}^\mathbf{1})p_t(\mathbf{x}^\mathbf{2} \mid \mathbf{z}^\mathbf{2})p_{\text{noise}}(\mathbf{x}^\mathbf{3})
\end{align}

The data prior $p_{\text{data}}(\mathbf{z}\mid \mathbf{c})$ factorizes according to the causal dependency (Assumption \ref{assump:causal}):
\begin{align}
p_{\text{data}}(\mathbf{z}\mid \mathbf{c}) &= p_{\text{data}}(\mathbf{z}^\mathbf{1} \mid \mathbf{c}^\mathbf{1})
p_{\text{data}}(\mathbf{z}^\mathbf{2} \mid \mathbf{z}^\mathbf{1},\mathbf{c}^\mathbf{1},\mathbf{c}^\mathbf{2}) \\
&\quad \times p_{\text{data}}(\mathbf{z}^\mathbf{3} \mid \mathbf{z}^\mathbf{1},\mathbf{z}^\mathbf{2},\mathbf{c}^\mathbf{1},\mathbf{c}^\mathbf{2}, \mathbf{c}^\mathbf{3})
\end{align}

We now compute the posterior mean $g_t(\mathbf{x}, \mathbf{c})$ for each region.

For Region 1 ($0:m$), since $p_t(\mathbf{x}^\mathbf{1} \mid \mathbf{z}^\mathbf{1})$ is a Dirac delta, we have:
\begin{equation}
g_t^\mathbf{1}(\mathbf{x},\mathbf{c}) = \mathbf{x}^\mathbf{1}.
\end{equation}

For Region 2 ($m:n$), the posterior depends on the history $\mathbf{z}^1 = \mathbf{x}^1$:
\begin{align}
g_t^\mathbf{2}(\mathbf{x},\mathbf{c})
&=\int \mathbf{z}^\mathbf{2}
\frac{p_t(\mathbf{x}^\mathbf{2}\mid \mathbf{z}^\mathbf{2})
p_{\text{data}}(\mathbf{z}^\mathbf{2}\mid \mathbf{z}^\mathbf{1}=\mathbf{x}^\mathbf{1}, \mathbf{c}^{\mathbf{1,2}})}
{p_t(\mathbf{x}^\mathbf{2} \mid \mathbf{z}^\mathbf{1} = \mathbf{x}^\mathbf{1}, \mathbf{c}^{\mathbf{1,2}})}
\mathrm{d}\mathbf{z}^\mathbf{2}
\end{align}
This term depends only on $\mathbf{x}^{1,2}$ and $\mathbf{c}^{1,2}$.

For Region 3 ($n:K$), the posterior mean involves an expectation over $\mathbf{z}^3$. Note that $p_t(\mathbf{x}^3 \mid \mathbf{z}^3) = p_{\text{noise}}(\mathbf{x}^3)$ is independent of $\mathbf{z}^3$. Thus:
\begin{align}
g_t^\mathbf{3}&(\mathbf{x},\mathbf{c})
= \int \mathbf{z}^\mathbf{3} \frac{p_t(\mathbf{x}\mid \mathbf{z})p_{\text{data}}(\mathbf{z}\mid \mathbf{c})}{p_t(\mathbf{x}, \mathbf{c})} \mathrm{d}\mathbf{z} \\
&= \int \left[ \int \mathbf{z}^\mathbf{3} p_{\text{data}}(\mathbf{z}^\mathbf{3} \mid \mathbf{z}^\mathbf{1}=\mathbf{x}^\mathbf{1}, \mathbf{z}^\mathbf{2}, \mathbf{c}) \mathrm{d}\mathbf{z}^\mathbf{3} \right] \\
&\quad \times \frac{p_t(\mathbf{x}^\mathbf{2}\mid \mathbf{z}^\mathbf{2}) p_{\text{data}}(\mathbf{z}^\mathbf{2}\mid \mathbf{z}^\mathbf{1} = \mathbf{x}^\mathbf{1}, \mathbf{c}^\mathbf{1,2})}{p_t(\mathbf{x}^\mathbf{2} \mid \mathbf{x}^\mathbf{1}, \mathbf{c}^\mathbf{1,2})} \mathrm{d}\mathbf{z}^\mathbf{2} \\
&= \int \mathbb{E}[\mathbf{z}^\mathbf{3}\mid \mathbf{z}^\mathbf{1}=\mathbf{x}^\mathbf{1}, \mathbf{z}^\mathbf{2}, \mathbf{c}] \\
&\quad \times \frac{p_t(\mathbf{x}^\mathbf{2}\mid \mathbf{z}^\mathbf{2}) p_{\text{data}}(\mathbf{z}^\mathbf{2}\mid \mathbf{z}^\mathbf{1} = \mathbf{x}^\mathbf{1}, \mathbf{c}^\mathbf{1,2})}{p_t(\mathbf{x}^\mathbf{2} \mid \mathbf{x}^\mathbf{1}, \mathbf{c}^\mathbf{1,2})} \mathrm{d}\mathbf{z}^\mathbf{2}
\end{align}
However, in this region, $\alpha_t^k=0$ and $\beta_t^k=1$, implying $\dot{\bm{\alpha}}_t = \mathbf{0}$ and $\dot{\bm{\beta}}_t = \mathbf{0}$. Consequently, the velocity coefficients $\mathbf{a}_t$ and $\mathbf{b}_t$ are zero, so the value of $g_t^3$ does not affect the velocity field.

Finally, we substitute these results into the velocity equation.
In Regions 1 and 3, the time derivatives vanish, yielding zero velocity.
In Region 2, the velocity is determined by the local posterior mean. Thus:

\begin{align}
u_t&(\mathbf{X}_t, \mathbf{c}^{0:K})\\
&= \begin{bmatrix}
\mathbf{0}^\mathbf{1} \\
\left(\dot{\bm{\alpha}}_t - \frac{\dot{\bm{\beta}}_t}{\bm{\beta}_t}\odot \bm{\alpha}_t\right) \odot g_t^\mathbf{2}(\mathbf{X}_t^\mathbf{1,2},\mathbf{c}^\mathbf{1,2}) + \left(\frac{\dot{\bm{\beta}}_t}{\bm{\beta}_t}\right) \odot \mathbf{X}_t \\
\mathbf{0}^\mathbf{3}
\end{bmatrix} \\
&= \begin{bmatrix}
\mathbf{0}^{0:m(t)} \\
u_t^{m(t):n(t)}(\mathbf{X}_t^{0:n(t)}, \mathbf{c}^{0:n(t)}) \\
\mathbf{0}^{n(t):K}
\end{bmatrix}
\end{align}

Regardless of the prediction parameterization, since the update $\mathrm{d}\mathbf{X}_t$ always reverts to the velocity $u_t$ (which is zero in regions 1 and 3), we are only concerned with the active window. Thus, it suffices to train the prediction target within the active window.

\end{proof}

\begin{table*}[thb]
\centering
\scalebox{0.77}{
\begin{tabular}{cccccccccccc}
& layers & hidden & ffn & heads & window & $v$/$\epsilon$/$x_{0}$ & steps & FID$\downarrow$ & R@3$\uparrow$ & MM-Dist$\downarrow$ & Diversity$\rightarrow$\\
\midrule
real motion & &  &  &  &  &  & - & 0.002 & 0.797 & 2.974 & 9.503 \\
ours  & 8 & 1024 & 2048 & 8 & 5 & $v$ & 10 & \et{\textbf{0.057}}{.002} & \et{0.810}{.003} & \et{2.887}{.007} & \et{9.579}{.062} \\
\midrule
& 2 & 1024 & 2048 & 8 & 5 & $v$ & 10 & 0.087 & 0.787 & 3.006 & 9.419 \\
& 4 & 1024 & 2048 & 8 & 5 & $v$ & 10 & 0.080 & 0.810 & 2.925 & 9.536 \\
& 12 & 1024 & 2048 & 8 & 5 & $v$ & 10 & 0.062 & 0.798 & 2.942 & \textbf{9.524} \\
& 8 & 128 & 512 & 8 & 5 & $v$ & 10 & 0.121 & 0.786 & 3.048 & 9.551 \\
& 8 & 256 & 1024 & 8 & 5 & $v$ & 10 & 0.077 & 0.800 & 2.954 & 9.531 \\
& 8 & 512 & 2048 & 8 & 5 & $v$ & 10 & 0.088 & 0.811 & 2.949 & 9.560 \\
& 8 & 768 & 3072 & 16 & 5 & $v$ & 10 & 0.084 & 0.812 & 2.936 & 9.584 \\
& 8 & 1024 & 4096 & 16 & 5 & $v$ & 10 & 0.073 & 0.814 & 2.909 & 9.504 \\
& 8 & 2048 & 8192 & 16 & 5 & $v$ & 10 & 0.081 & 0.792 & 3.001 & 9.561 \\
\midrule
& 8 & 1024 & 2048 & 8 & 1 & $v$ & 10 & 1.38 & 0.644 & 3.987 & 8.855 \\
& 8 & 1024 & 2048 & 8 & 20 & $v$ & 10 & 0.060 & \textbf{0.821} & \textbf{2.842} & 9.683 \\
\midrule
& 8 & 1024 & 2048 & 8 & 5 & $\epsilon$ & 10 & 0.124 & 0.807 & 2.855 & 9.549 \\
& 8 & 1024 & 2048 & 8 & 5 & $x_{0}$ & 10 & 0.060 & 0.805 & 2.901 & 9.525 \\
\midrule
& 8 & 1024 & 2048 & 8 & 5 & $v$ & 20 & 0.061 & 0.807 & 2.903 & 9.624 \\
& 8 & 1024 & 2048 & 8 & 5 & $v$ & 100 & 0.060 & 0.812 & 2.895 & 9.653 \\
\end{tabular}
}
\caption{Ablation study of the model architecture on HumanML3D test set.} 
\label{tab:parameter}
\end{table*}

\begin{remark}[Stochastic Extension]
If we introduce a diffusion term $u^{fp}_t = u_t + \frac{\bm{\sigma}_t^2}{2} \odot s_t$, the drift $u^{fp}_t$ is no longer guaranteed to be zero in the finalized and future regions. However, the score function $s_t$ in these regions depends only on $\mathbf{X}_t$:
\begin{itemize}
\item In the future region ($k \ge n(t)$), we have $\alpha_t^k=0, \beta_t^k=1$. The posterior mean vanishes, so substituting into the score definition yields $s_t(\mathbf{x}) = -\mathbf{x}$.
\item In the finalized region ($k < m(t)$), we substitute the posterior mean $g_t(\mathbf{x}) = \mathbf{x}$ derived in the proof. Keeping $\alpha_t$ and $\beta_t$ (treating $\beta_t$ as a small non-zero value to avoid division by zero), the marginal score becomes
\begin{align}
s_t(\mathbf{x}) &= \frac{\bm{\alpha}_t \odot \mathbf{x} - \mathbf{x}}{\bm{\beta}_t^2} \\
&= -\frac{(1-\bm{\alpha}_t)\odot \mathbf{x}}{\bm{\beta}_t^2} = -\frac{\mathbf{x}}{\bm{\beta}_t}
\end{align}
\end{itemize}
Since we generally do not wish to add diffusion to parts that are already denoised or remain pure noise, and to avoid numerical instability where $\beta_t \approx 0$, we set $\bm{\sigma}_t$ to be a diagonal matrix that is non-zero only in the active window. This preserves the streaming locality. Within the active window, $s_t$ can be derived from $u_t$.
\end{remark} 

\begin{remark}[Interpretation of Causal Dependency]
This assumption (Assumption~\ref{assump:causal}) does not imply that the model cannot plan or execute complex behaviors. The available control signal $\mathbf{c}^{0:l}$ can itself contain complex, long-term instructions (\textit{e.g.}, "first walk, then run"). The condition merely states that the motion generated up to frame $l$ does not depend on \textit{future, unseen} instructions that arrive later in the stream. Thus, this assumption holds in most practical streaming scenarios. We explicitly use this assumption to factorize $p_{\text{data}}(\mathbf{z}\mid\mathbf{c})$ in the proof of Theorem~\ref{thm:streaming_locality_appendix}.
\end{remark} 

\section{Baseline Implementation Details}
\label{sec:sup3}
To ensure a fair comparison, we modified the PRIMAL baseline as follows: (1) replaced the cosine schedule-based diffusion training with standard flow matching; (2) changed the transformer backbone from \texttt{nn.TransformerEncoderBlock} to the DiT block used in Wan; (3) removed the FK and velocity losses; (4) upgraded the text conditioning from discrete action tags (\textit{e.g.}, ``walk'', ``run'') to natural language descriptions using a T5 encoder; and (5) adopted the same 263D motion representation instead of the original 267D format.

\section{Hyper Parameter Search}
\label{sec:sup4}
We conduct a grid search over the main hyperparameters of our motion generation network, and report the results in Table \ref{tab:parameter}. A hidden size of 1024 yields the best performance in our setting. Increasing the window size provides slight gains, but at the cost of higher response latency. Different prediction types do not show significant performance differences under our configuration.

\begin{table*}[h]
\centering
\scalebox{0.85}{
\begin{tabular}{l ccc}
\toprule
\multirow{2}{*}{\textbf{Comparison (A vs B)}} & \textbf{Preference} & \textbf{Transition} & \textbf{Consistency} \\
 & Win Rate (A : B) & Win Rate (A : B) & Win Rate (A : B) \\
\midrule
\multicolumn{4}{l}{\textit{Ours vs. Baselines}} \\
FloodDiffusion vs PRIMAL & 63.2\% (12 : 7) & 62.5\% (10 : 6) & 55.6\% (10 : 8) \\
FloodDiffusion vs MotionStreamer & 56.3\% (9 : 7) & 54.5\% (6 : 5) & 50.0\% (6 : 6) \\
\midrule
\multicolumn{4}{l}{\textit{Ours vs. Real Motion}} \\
FloodDiffusion vs GT & 50.0\% (8 : 8) & 47.1\% (8 : 9) & 42.1\% (8 : 11) \\
\midrule
\multicolumn{4}{l}{\textit{Reference Comparisons}} \\
PRIMAL vs GT & 31.8\% (7 : 15) & 33.3\% (5 : 10) & 42.9\% (6 : 8) \\
MotionStreamer vs GT & 28.6\% (4 : 10) & 40.0\% (8 : 12) & 39.1\% (9 : 14) \\
MotionStreamer vs PRIMAL & 61.5\% (8 : 5) & 52.4\% (11 : 10) & 57.1\% (8 : 6) \\
\bottomrule
\end{tabular}
}
\caption{\textbf{Pairwise comparison results.} We report the win rate of method A over method B, along with the raw vote counts in parentheses. Our method outperforms both baselines on all metrics (or ties) and achieves a win rate close to 50\% against real motion (GT), indicating high realism.}
\label{tab:pairwise_counts}
\end{table*}

\section{Details of User Study}
\label{sec:sup5}

To evaluate the perceptual quality of the generated motions, we conducted a user study with 100 participants. We compared our \textit{FloodDiffusion} against real motion ground truth (GT) and two streaming baselines: \textit{PRIMAL} \citep{zhang2025primal} and \textit{MotionStreamer} \citep{xiao2025motionstreamer}.

\paragraph{Questionnaire Design}
We collected 100 questionnaires in total. Each questionnaire consists of three distinct questions, where each question compares a randomly sampled pair of videos generated by two different methods. The three questions correspond to the three evaluation metrics respectively:
\begin{enumerate}
    \item \textbf{Preference}: Given a pair of videos, choose the one that appears more reasonable and plausible given the text prompt.
    \item \textbf{Transition}: Given a pair of videos, choose the one that transitions more smoothly between different actions.
    \item \textbf{Consistency}: Given a pair of videos, choose the one that better maintains a consistent motion style across different actions.
\end{enumerate}
An example of the question interface is shown in Figure \ref{fig:questionnaire_example}.

\begin{figure}[h]
    \centering
    \fbox{
    \begin{minipage}{0.95\linewidth}
        \vspace{0.5em}
        \textbf{\large User Study - Question 1 (Preference)}
        \vspace{0.5em}
        
        \textbf{Prompt Sequence:} ``\textit{A person walks forward}'' $\rightarrow$ ``\textit{A person sits down}''
        
        \vspace{1em}
        \textbf{Question:} Which video appears more reasonable given the text prompt?
        \vspace{0.5em}
        
        \begin{center}
        \begin{tabular}{c@{\hspace{0.5cm}}c}
             \fbox{\parbox{2.8cm}{\centering \vspace{1.2cm} \textbf{Video A} \vspace{1.2cm}}} &  
             \fbox{\parbox{2.8cm}{\centering \vspace{1.2cm} \textbf{Video B} \vspace{1.2cm}}} \\
             \noalign{\vspace{0.5em}}
             \parbox{2.8cm}{\centering \Large$\bigcirc$\normalsize~Select A} & 
             \parbox{2.8cm}{\centering \Large$\bigcirc$\normalsize~Select B} \\
        \end{tabular}
        \end{center}
        
        \vspace{0.5em}
    \end{minipage}
    }
    \caption{An illustrative example of a single question in the user study. Each questionnaire contains three such comparisons, one for each metric (Preference, Transition, and Consistency), using different video pairs.}
    \label{fig:questionnaire_example}
\end{figure}

\paragraph{Pairwise Comparison Results}
We aggregated the votes from all valid questionnaires. The detailed head-to-head win counts for each pair of methods are reported in Table \ref{tab:pairwise_counts}. These raw counts were used to compute the Bradley--Terry scores presented in the main text.